\documentclass[draftcls,onecolumn,12pt]{IEEEtran}
\usepackage[]{authblk}
\usepackage{amsmath,amstext,amsthm,amssymb,epsf}
\usepackage{latexsym} 
\usepackage[pdftex]{graphicx}
\graphicspath{ {./figures}}
\usepackage{setspace} 
\usepackage{float}
\usepackage[section]{placeins}
\usepackage[]{hyperref}
\usepackage{caption} 
\usepackage{xcolor}
\usepackage{subcaption}
\usepackage[]{soul}
\usepackage[style=alphabetic]{biblatex}
  
\hypersetup{
    colorlinks=true,
    linkcolor={blue},
    citecolor={blue!50!black},
    urlcolor={blue!80!black}
}
\numberwithin{equation}{section} 

\newtheorem{theorem}{\sc Theorem}

\newtheorem{coro}{\sc Corollary}
\newtheorem{req}{\sc Requirement}

\newtheorem{defin}{\sc Definition}
\newtheorem{rem}{\sc Remark}
\newtheorem{cla}{\sc Claim}
\newtheorem{ex}{\sc Example}

\newenvironment{remark}{\begin{rem}}{\hspace*{\fill}$\Diamond$\end{rem}}

\newcommand\numberthis{\addtocounter{equation}{1}\tag{\theequation}}

\begin{document}

\title{A Kolmogorov metric embedding for live cell microscopy signaling patterns}
\author[1]{Layton Aho}
\author[1]{Mark Winter}
\author[1]{Marc DeCarlo}
\author[2]{Agne Frismantiene}
\author[2]{Yannick Blum}
\author[2]{Paolo Armando Gagliardi}
\author[2]{Olivier Pertz}
\author[1*]{Andrew R. Cohen}
\affil[1]{Electrical and Computer Engineering, Drexel University, USA}
\affil[2]{Institute of Cell Biology, Univ. of Bern, Switzerland}
\affil[*]{correspondence to andrew.r.cohen@drexel.edu}

\maketitle
\begin{abstract}
    We present a metric embedding that captures spatiotemporal patterns of cell signaling dynamics in 5-D $(x,y,z,channel,time)$ live cell microscopy movies. The embedding uses a metric distance called the normalized information distance (NID) based on Kolmogorov complexity theory, an absolute measure of information content between digital objects. The NID uses statistics of lossless compression to compute a theoretically optimal metric distance between pairs of 5-D movies, requiring no \emph{a priori} knowledge of expected pattern dynamics, and no training data. The cell signaling structure function (SSF) is defined using a class of metric 3-D image filters that compute at each spatiotemporal cell centroid the voxel intensity configuration of the nucleus w.r.t. the surrounding cytoplasm, or a functional output \emph{e.g.} velocity. The only parameter is the expected cell radii ($\mu m$). The SSF can be optionally combined with segmentation and tracking algorithms. The resulting lossless compression pipeline represents each 5-D input movie as a single point in a metric embedding space. The utility of a metric embedding follows from Euclidean distance between any points in the embedding space approximating optimally the pattern difference, as measured by the NID, between corresponding pairs of 5-D movies. This is true throughout the embedding space, not only at points corresponding to input images. Examples are shown for synthetic data, for 2-D+time movies of ERK and AKT signaling under different oncogenic mutations in human epithelial (MCF10A) cells, for 3-D MCF10A spheroids under optogenetic manipulation of ERK, and for ERK dynamics during colony differentiation in human stem cells.

\end{abstract}

\section{Introduction}
A key open challenge in systems biology is to elucidate and manipulate cellular / environmental interaction mechanisms. For example, ERK and AKT are both kinases that feature prominently in the control of basic cellular functions such as motion, cell cycle, \emph{etc.} \supercite{Gagliardi2021,Gagliardi2023,Ender2022,Purvis2013}. The processes of interest are  high-dimensional (5-D) movies, taking place in 3-D space over time across multiple imaging channels. Here we present a new unsupervised metric embedding technique based on Kolmogorov complexity theory \supercite{Vitanyi2008} that represents each movie as a single point in a low dimensional space that preserves \emph{``any and all differences''} \supercite{Vitanyi2005} among the patterns in the input movies. 

Any computational analysis of live cell and tissue microscopy movies can be thought of as an \emph{embedding} that represents each movie by a point in some lower dimensional space. Consider an example embedding that represents each movie using the average voxel intensity across all frames. Each embedded movie would be a point on a 1-D line, with brighter movies towards the right and dimmer movies towards the left. Given two movies with average intensity of $I_1$ and $I_2$ respectively, construct a third movie with  intensity $I_3 = \frac{I_1 + I_2}{2}$. This example embedding is a \emph{metric embedding} meaning that $I_3$ will be embedded as a point exactly halfway between the points representing $I_1$ and $I_2$.  A second example embedding could train a deep learning network to classify movies as either bright or dark based on some ground truth and use the output values of individual neurons as dimensions of the embedding. With this embedding, $I_3$ will not be located halfway between the points representing $I_1$ and $I_2$, and we say that the embedding is non-metric. A metric embedding is required to measure and infer relationships between points in the embedded space such that the corresponding distance relationship is valid in the input space. For example, in section \ref*{ERK_signaling_HSC} we use the embedded space representation to determine if particular genetic manipulations significantly change the relationship between ERK signalling and the resulting cellular velocity. This analysis using spatial characteristics of the embedding representation to measure characteristics of the input movies is only possible because of the metric properties of the embedding. The embedding approach here is based on a metric distance based on Kolmogorov compression computed between pairs of 5-D movies. 

Kolmogorov complexity theory \supercite{Vitanyi2008} defines the \emph{information distance} \supercite{Vitanyi2009}, a perfect distance between any two digital objects, as the length in bytes of the shortest Turing machine program that outputs one object given the other as input. The information distance is universal in that it considers every possible Turing machine program as a compression, and is unsupervised and parameter free.  In theory, any point in the embedded space (not just the points corresponding to our input movies) would correspond optimally to the \emph{``true''} movie, the single movie from among all possible movies that corresponds to the given point in the embedding space. The information distance is not computable, and must be approximated. The normalized compression distance (NCD) is a metric distance that provides an effective approximation to the information distance\supercite{Vitanyi2005} using lossless compression. Throughout the remainder, we use NID and NCD interchangeably, with NID typically denoting the theoretical distance based on Kolmogorov complexity and NCD denoting the approximation using lossless 3-D image compression.

A key enabling technology for the present work is the Free Lossless Image Format (FLIF), a lossless 3-D image compression algorithm \supercite{FLIF}. We use the NCD with FLIF to find patterns of spatiotemporal visual similarity between 5-D movies. The key steps are shown in Figure \ref*{fig:figure0}. The input is a collection of 5-D microscopy movies, $(x,y,z,channel,time)$.  The movies are denoised and optionally segmented, tracked and lineaged (section \ref*{sect.SegTrack}).  Metric structure enhancing filters apply \emph{a priori} knowledge prior to compression. One contribution of the present work is the cell signaling structure function that combines the notion of Kolmogorov structure functions with metric blob enhancing filters \supercite{FRANGI}. The output of these metric image enhancing filters is input directly to the compression algorithm for 2-D+time movies, for 3-D+time movies we project the output to two spatial dimensions plus time. Given $N$ input movies, we compute the $NxN$ pairwise NCD distance matrix, and use the eigenvectors of that matrix to define our embedding space.

The approach described here results in a particularly desirable type of metric embedding known as a \emph{Reproducing Kernel Hilbert Space}\supercite{Manton2015} (RKHS). An RKHS embedding maintains Euclidean distance between all points in the embedding space w.r.t. the kernel distance on the image space. For example, with an RKHS the magnitude and direction of the vector between any two points in the embedding space is proportional to the vector between the corresponding points in the input image space. The maximum dimension of this RKHS is equal to the number of input movies. RKHS embeddings differ from other embedding approaches in that they preserve this so-called ``extrinsic geometry'' \supercite{Manton2015} at all points in the embedding space, not just the points corresponding to the input images. 

We propose the normalized information distance (NID) as the reproducing kernel for a Hilbert space embedding. In theory, this NID kernel captures all meaningful differences between image pairs to define a pattern space throughout the embedding, optimal in the mean square error sense of preserving dimensions corresponding to the largest eigenvalues of the pairwise distance matrix. In addition to being an optimal distance measure between arbitrary 5-D patterns, the underlying Kolmogorov complexity theory provides \emph{algorithmic structure functions}, metrics for how meaningful are the representations of the input or the embedded data based on compression characteristics (section \ref*{Sect.computeSSF}). Here we use the unsupervised cluster structure function \supercite{Cohen2023} as a scalar measure or indicator of how \emph{meaningful} a particular embedding, (section \ref*{Sect.CSF_methods}). For applications where labeled data is available, a key advantage of the RKHS is as a basis for subsequent supervised machine learning algorithms \supercite{Manton2015,Cohen2009,NM_cellFate,Cohen2023}. Training supervised learners in the RKHS results in reduced complexity and improved generalization compared to learning machines in the input image space \supercite{Manton2015}. This makes the RKHS especially well-suited for subsequent supervised (semi-supervised spectral learning \supercite{Kamvar}) or unsupervised (spectral clustering \supercite{Jordan}) machine learning algorithms. Throughout the remainder, the use of this NID kernel with the FLIF 3-D image compression is referred to as a Kolmogorov embedding of the 5-D input movies.

\begin{figure*}[!ht]
    {\includegraphics[width=1.0\textwidth]{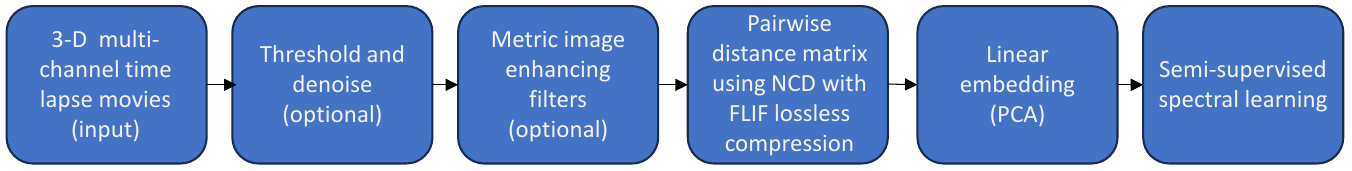}}
    \centering
    \caption[]{\textbf{The normalized information distance is a metric embedding kernel for spatiotemporal patterns of cell signaling.} Start with $N$ 2-D or 3-D multichannel time-lapse live cell microscopy movies. Threshold and denoise the movies, optionally segmenting, tracking and lineaging. Metric structure enhancing image filters compute cell signaling state, or functional outputs such as velocity. The cell signaling state is quantified by the intensity of the nuclear pixels w.r.t the surrounding cytoplasm. 3-D image compression is then used as a normalized pairwise distance metric between structure enhanced movies. The resulting distance matrix defines an optimal embedding based on visual differences among the input movies as captured by the normalized compression distance (NCD). Each input movie is represented as a single point in the embedding space. Importantly, all points in the embedded space, not just the ones corresponding to the input movies, optimally represent the pattern characteristics of a corresponding ``true'' input image. Supervised or unsupervised learning algorithms on the kernel embedding space are less complex and more generalizable \supercite{Manton2015}. No training data or other \emph{a priori} knowledge is required, although it can be utilized if available.}
    \label{fig:figure0}
  \end{figure*}

Figure \ref*{fig:figure1} shows a sample 2-D image frame from the ERK-KTR channel (A). This image is taken from a movie of cells with a mutation (PIK3CA\_H1047R) that exhibits a distinctive, visually obvious pulsing throughout the tissue monolayer. ERK-KTR activation increases as the nucleus darkens w.r.t. surrounding cytoplasm. We denoise the images, and then compute the cell signaling structure function (SSF) using metric blob enhancing filters. The output of the SSF is a 3-D image with the SSF response at each $(x,y,time)$ cell centroid (B). Pairs of these 3-D SSF output images are input to the FLIF compression. A 2-D projection of the 3-D SSF (C) provides an effective visualization of the ERK pulsing dynamics as diagonal yellow lines. Finally, the RKHS embedding of the 147 input movies from six different genetic conditions, each associated with distinctive patterns of ERK signaling dynamics \supercite{Jacques2021}, is shown in (D). Spectral learning theory tells us that because there are six different mutations in the input data the six principal dimensions of the RKHS should be preserved\supercite{Jordan}. To visualize the data, dimensionality must be reduced to 3-D (or less). The RKHS formulation automatically selects the  principal dimensions that optimally preserve the characteristics of the input image patterns, enabling improved visualization in the lower dimensional space. 

\begin{figure*}[!ht]
    \includegraphics[width=1.0\textwidth]{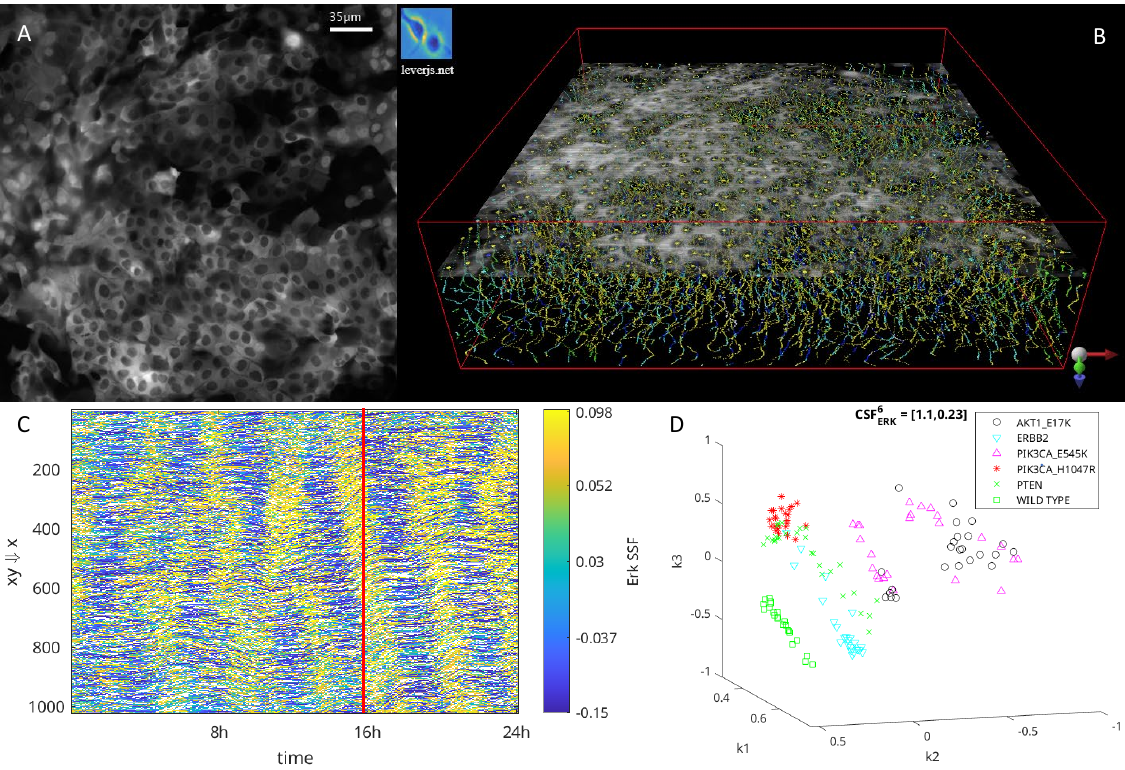}
    \centering 
    \caption[]{\textbf{A Kolmogorov embedding of live cell and tissue microscopy movies enables improved visualization and quantification of spatiotemporal patterns of cell signaling}. Time-lapse microscopy captured 147 movies from 6 different conditions, known to differ in ERK signaling dynamics, in 2-D monolayers of human breast epithelial (MCF10A) cells (A). In (B), we show the output of the cell signaling structure function (SSF) from the beginning of the movie to the image frame shown in (A), overlaid in gray. Each colored dot along the gray image represents the cell signaling activation at that location and time.  A maximum intensity projection of the 2-D+time SSF output (B) to a 1-D+time image (C) allows visualization of the waves of ERK signaling activation propagating throughout the tissue as diagonal yellow stripes. The red line in (C) indicates time of the image frame in (A, B). The Kolmogorov embedding (D) uses the normalized information distance with 3-D image compression to define a pattern space that reflects everywhere the similarity characteristics visually observable in the input movies. The cluster structure function (CSF) reported in (D) is an unsupervised measure of how meaningfully the embedding represents the six different ground truth classes, reported as [mean, standard deviation] of the per-cluster optimality deficiency for ERK signaling in the 6-D RKHS. See also Supplementary Movie \ref*{Movie1}. 
  }  
  \label{fig:figure1}
  \end{figure*} 

The introduction continues below with a description of the cell signaling structure function, and a very brief review of the related literature. Section \ref*{Sect.Results} details the results of our analysis for 2-D monolayers of human breast epithelial (MCF10A) cells from six different oncogenic mutations associated with distinctive changes in cell signaling patterns, for human stem cell colonies under self-renewing and differentiating conditions, for optogenetic excitation of MCF10A cells cultured in 3-D synthetic breast spheroids, and for a synthetic phantom dataset. Section \ref*{Sect.Discussion} gives some concluding observations on the approach and suggests avenues for future research. Section \ref*{Sect.Methods} details the methods for live cell imaging, for segmentation and tracking, for computing and quantizing the SSF, and calculating the cluster structure function in the RKHS.
\FloatBarrier 
\subsection{The cell signaling structure function}
We define a structure function as any metric function that reports how well a model fits the input data. This is equivalent to the Kolmogorov structure function that models any digital object $x$ using a finite set $S$ containing $x$, defined as $H(x)=\min\limits_{S}log(|S|),x \in S$. The idea is that if there are no ``\emph{simple special properties}'' \supercite{Vitanyi2004} that would allow $x$ to be specified within $S$ more efficiently than encoding its ordinal within the set then $S$ is an optimally meaningful representation for $x$. This definition of structure functions also admits the metric blob-, tube- and plate- enhancing filters introduced by Frangi et al. \supercite{FRANGI} where a Gaussian smoothing followed by a Hessian shape enhancement achieved remarkable results at detecting tube-like structures in noisy images. We build on this approach to introduce a metric blob-enhancing filter to capture cell signaling state for arbitrary nuclear markers in live cell microscopy.

We propose the cell signaling structure function (SSF) to quantify  the signaling state of each cell at each time and imaging channel.  The cell signaling state is defined by the intensity of the nuclear voxels w.r.t. surrounding cytoplasmic voxels. The nuclear intensity is defined by $n$, the surrounding cytoplasmic intensity is defined by $c$, with $c,n \in [0.0,1.0]$. The SSF is a 2-D vector valued function. The first dimension is non-zero when $c>n$, the second dimension is non-zero when $n>c$. We take as the model $M$ the case $c = n$, an absence of signal, and the SSF is defined as 
\begin{align*}
  \label{eqn.SSF_2D}
  \numberthis
  H_{SSF} &= <c-n,0> & c > n \\
&= <0,n-c> & c < n \\
&= <0,0> & c = n.
\end{align*}

The SSF starts at $0$ when the model fits perfectly with no signaling activity ($c = n$) and increases to its maximum value as signaling activity increases to $c = 1, n = 0$. This definition differs from both the original Kolmogorov structure function and the recently proposed cluster structure function\supercite{Cohen2023} in that is 2-dimensional. Another difference between the SSF compared to the Kolmogorov structure function and the CSF is that the SSF is not intended to optimize model parameters but rather used directly as a measure of model fit indicating cell signaling activation levels. A third difference is that the SSF is normalized to $[0.0,1.0]$ in each dimension. For convenience, we encode $H_{SSF}$ into a signed scalar value $\in [-1.0,1.0]$, 

\begin{equation}
  \label{eqn.SSF_1D}
  H_{SSF} = c-n.
\end{equation}
We use the formulation in \ref*{eqn.SSF_1D} throughout the remainder. Our implementation of the SSF uses a Laplacian of Gaussian (LoG) blob enhancing filter \supercite{HIP} whose response scales naturally to $[-1.0,1.0]$, to robustly approximate the $c-n$, as described in Section \ref*{Sect.computeSSF}. Equivalently, $C$ and $N$ could be measured using cell segmentation algorithms as in previous work \supercite{Ender2022,Gagliardi2021,Gagliardi2023}, although that is a more complex approach.

In order for an embedding of cell signaling dynamics to be considered metric, the function that computes the signaling state of the cell from the nuclear and cytoplasmic voxel intensities must be a metric function. For example, kinase translocation reporters (KTRs) \supercite{Regot2014,Kudo2017} are biosensors that offer improved imaging of multiple kinases simultaneously in living cells. The biosensor uses the concentration of the kinase in the nucleus \emph{vs.} the cytoplasm as a measure of signal activation. Current approaches to measuring activation of KTR biosensors such as the cytonuclear ratio $\frac{C}{N}$ \supercite{Regot2014} and its variants, \emph{e.g.} $\frac{C}{C+N}$, are non-metric and cannot be used in a metric embedding framework. A key risk with non-metric embeddings is that populations of input objects from different conditions can be biased with different distance characteristics \emph{w.r.t.} other  populations. Informally, if we perturb some given $(C,N)$ value by a fixed amount, we require the resulting activation level to also be perturbed by some fixed amount no matter what values $(C,N)$ take and this is not possible with the cytonuclear ratio and its non-metric variants. 

Figure \ref*{fig:ssf} shows the SSF in comparison to the cytonuclear ratio currently used to quantify KTR activation levels. Using a phantom image (A), we simulate KTR activation (B), with the SSF activation level (black line) precisely matching the ground truth activation percentage (yellow line). The magenta line represents the unnormalized LoG response, the red line shows the cytonuclear response, the green line shows the response for $\frac{C}{C+N}$. The cytonuclear ratio requires careful human supervision to choose the response range for KTR signal quantification. The ratio $\frac{C}{C+N}$ improves somewhat the cytonuclear ratio, but shows increasing error bias as the nucleus brightens. The SSF recovers exactly the true underlying signaling state throughout the sample activation region. In (C), the mapping between cytoplasmic and nuclear intensity is shown for the cytonuclear ratio, and in (D) for the SSF. Note how the linearity and normalization of the SSF enhances the color space mapping considerably compared to the cytonuclear ratio. For the 2-D monolayer movies of human breast epithelial (MCF10A) cells we compared the CSF results obtained from the SSF on the ERK channel to the results obtained using the cytonuclear ratio, and found the SSF resulted in a significantly lower CSF value indicating that the SSF extracted significantly more structure or meaningful information compared to the cytonuclear ratio ($p<3e-10$), see Section \ref*{sect.mcf10a_results}.

\begin{theorem}
  \label{theorem.SSF}
$H_{SSF}$ is a positive semi-definite function, making it an optimal mapping from the 2-D $(cytoplasmic,nuclear)$ intensity space to a single 1-D activation value. 
\end{theorem}
\begin{proof}
Consider the case $c \geq n$ in Eqn. \ref*{eqn.SSF_2D}, given $c \in [0.0,1.0], n \in [0.0,c]$, $c-n$ is equivalent to the Euclidean distance between the 1-D point $(c-n)$ and the point $(0)$ and the result follows from the Euclidean distance being positive semi-definite. Combine the case $c<n$ using the sum of two positive semi-definite functions is positive semi-definite, making $H_{SSF}$ a reproducing kernel Hilbert space (RKHS) embedding of the $(c,n)$. The optimality of the mapping from $(c,n) \mapsto H_{SSF}$ follows from the RKHS properties \supercite{Manton2015}.
\end{proof}
\begin{remark}
Being a positive semi-definite function is equivalent to being a metric distance function. The cell signaling structure function can then be interpreted as a metric distance between the model of no signaling activation ($c=n$) and our given $(c,n)$ configuration. Being positive semi-definite seems a necessary condition for any structure function, but this has not been expressed explicitly in previous structure function definitions \supercite{Cohen2023,Vitanyi2004,Vitanyi2006}.
\end{remark}
\begin{figure*}[!ht]
    \includegraphics[width=1.0\textwidth]{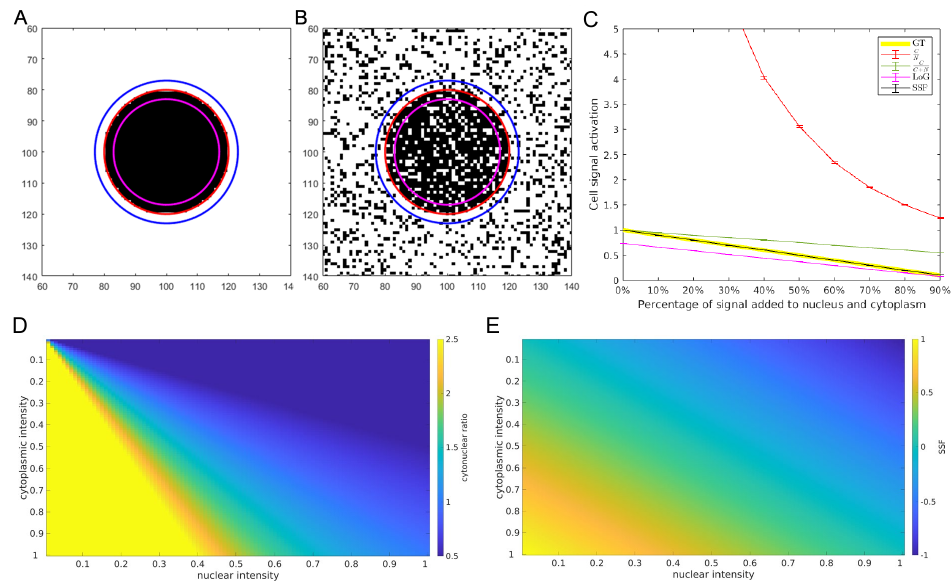} 
    \centering
    \caption[]{\textbf{The cell signaling structure function (SSF) is a metric function that combines nuclear \emph{vs.} cytoplasmic voxel intensity to form a 1-D scale space representation that peaks at the cell centroid.} Phantom image (A) shows dark nucleus against bright background. This represents a cell that is fully activated w.r.t. a KTR signal. The region between the red and magenta circles is used to estimate the nuclear intensity, the region between the red and blue circles is conventionally used to estimate the cytoplasmic intensity. Adding 40\% shot noise to the phantom image (B) such that ~20\% of the nuclear pixels are white and ~20\% of the cytoplasmic pixels are black simulates KTR signal decrease, representing a cell at 60\% activation. Varying the activation levels (C) shows the SSF (black line) accurately captures KTR signal activation across variations in nuclear and cytoplasmic intensity. The activation measure in (C) is shown with the ground truth (yellow line) varying linearly from 0.0 to 1.0 as simulated signal is added to the nucleus and subtracted from the cytoplasm. The Laplacian of Gaussian (LoG) (C) (magenta line) approximates the SSF without requiring explicit segmentation. Error bars in (C) represent standard error over 100 trials. Unlike the cytonuclear ratio ($\frac{C}{N}$), the SSF varies linearly from [0,1] as the KTR moves from nucleus to cytoplasm. The related ratio $\frac{C}/{C+N}$ is better behaved compared to the cytonuclear ratio, but the error rate increases as the nucleus brightens. Because both $\frac{C}{N}$ and $\frac{C}{C+N}$ are non-metric functions they cannot be used in a metric embedding. (D) illustrates the non-metricity of the cytonuclear ratio, where a fixed step size in varying locations and directions results in very different changes in the output activation level. The SSF response (E) varies linearly with nuclear and cytoplasmic intensity and is bounded on $[-1.0,1.0]$.}
    \label{fig:ssf}
  \end{figure*} 
\subsection{Related Literature}
The two key contributions presented here are the cell signaling structure function (SSF), and the use of the metric normalized compression distance (NCD) \supercite{Vitanyi2005} as a reproducing kernel for a Hilbert space embedding of live cell signaling movies.  The SSF is a metric for cell signaling activity from kinase translocation reporters (KTRs). KTRs are fluorescent reporters that utilize nucleocytoplasmic shuttling to measure kinase activity in a single cell\supercite{Regot2014,Kudo2017}. KTRs have been developed for many signaling pathways, such as ERK (ERK-KTR), AKT (Fox01-FP), and Cd2k (DBH)\supercite{Kudo2017}, a cell cycle indicator. The current approach for measuring the estimated signal activity from the KTR is to calculate the ratio of cytoplasmic versus nuclear fluorescence intensity ($\frac{C}{N}$ ratio) from the reporter \supercite{Regot2014,Kudo2017,Gagliardi2021,Jacques2021,Ender2022,Gagliardi2023}. The nuclear to cytoplasmic ratio $\frac{N}{C}$ and variants like $\frac{C}{C+N}$ have also been proposed. All of these are non-metric functions. This means when the signal is perturbed, the magnitude of the response is dependent upon the direction within the $(C,N)$ space, violating the triangle inequality\supercite{Theodoridis2009}. In contrast, the SSF is an optimal 1-D representation of the 2-D $(cytoplasmic,nuclear)$ intensity space as shown in Theorem \ref*{theorem.SSF}.
 
Calculating the $\frac{C}{N}$ ratio requires the determination of both the cytoplasmic ($C$) and nuclear ($N$) intensities for a given cell. This is typically done by estimating the boundary between nucleus and cytoplasm, and then estimating the voxels intensities interior ($N$) and exterior ($C$) as described above and shown as the ring around the nucleus in Figure \ref*{fig:ssf} (A). The implementation of the SSF presented here (as outlined in \ref*{sect.computeSSF}) only requires a centroid and is more forgiving of less accurate boundary estimations. It is easier to find cell centroids than to compute an accurate boundary, as can be seen from the higher detection vs. segmentation accuracy on the Cell Tracking Challenge datasets \supercite{NM_ctc}. Given the nuclear and cytoplasmic intensities for computing $\frac{C}{N}$ or $\frac{N}{C}$, it is straightforward to compute the SSF, or the SSF can be approximated using the Laplacian of Gaussian imaging filter as in the present work. 

The Kolmogorov embedding approach described here is unique in its unsupervised, metric RKHS representation. The most widely used alternatives to this approach include neural network latent spaces and / or non-metric embedding methods like t-Stochastic Neighbor Embedding (t-SNE) and UMAP. Neural networks are inherently non-metric due to the activation function at the perceptron. Visualizing a neural network latent space as a method for understanding how the network is partitioning the data may be qualitatively interesting, but any measurement taken on such an embedding is not valid w.r.t. the properties of the input data. Other approaches include CODEX, finding similar patterns of cell signaling from individual cell trajectories of cytonuclear ratios input to a convolutional neural network (CNN) and the resulting CNN features are embedded using t-SNE for visualization and classification\supercite{Jacques2021}. Another approach uses the ARCOS algorithm, segmenting and tracking collective signaling events by thresholding cytonuclear intensities against a minimum size of collective events\supercite{Gagliardi2023}. These approaches are intended to identify specific types of signaling events and are useful in combination with the quantitative embedding for cell signaling pattern discovery techniques proposed here.


\section{Results}
\label{Sect.Results}
\subsection{ERK and AKT signaling in 2-D+time monolayer of human breast epithelium}
\label{sect.mcf10a_results}
Live movies of human breast epithelial cell monolayers were captured in six different imaging experiments. Each experiment captured four or five movies from each of wildtype cells plus five different mutations for a total of 147 movies. The movies contained an H2B nuclear marker, as well as ERK and AKT KTR reporters on three separate channels. The movies were captured at 5 minutes per frame for 24 hours. Each frame was $1024x1024$ pixels. Figure \ref*{fig:figure1} shows an example image frame from this data (A), the 2-D+time SSF with a single image frame in gray (B), the 1-D+time spatial maximum intensity projection as a 2-D pattern visualization (C) and the embedding of the 147 movies as a 3-D RKHS or pattern space (D). The mutations chosen for this application exhibit distinct patterns of ERK signaling \supercite{Jacques2021}, as seen in (C) as diagonal yellow stripes of ERK activation across the monolayer. These distinct signaling patterns are captured by the FLIF compression as can be seen by the segregation of the representative points for each movie in the RKHS (D). We compute the cluster structure function (CSF) in the RKHS embedding, preserving the principal six dimensions of the RKHS embedding as there are six ground truth classes. The CSF is computed against the ground truth classes, in the RKHS as the distance from each point to the centroid of the mutation associated with that movie, as in Section \ref*{Sect.CSF_methods}. For ERK, the CSF is $[1.1,0.23]$ (the first value is the mean CSF across the 147 movies, the second value is the standard deviation). For AKT, the CSF is $[1.3,0.18]$, significantly higher ($p<1e-10$) compared to the ERK CSF, indicating that AKT is significantly less informative compared to ERK w.r.t. the observed signaling patterns.

The SSF can also be used with segmentation and tracking algorithms, as in Section \ref*{sect.SegTrack}, to capture cellular velocity patterns by generating an image with the average velocity for each cell from time $t-1$ to $t$ and $t$ to $t+1$ stored at each $(x,y,t)$ cell centroid. The velocity is normalized to the maximum cellular velocity gate value used in the tracking algorithm (\ref*{sect.SegTrack}). For each of the 147 movies, we compute the NCD between the ERK SSF and the velocity SSF. Figure \ref*{fig:veloPlot} (A) shows the resulting NCDs between the ERK and velocity SSF as a multiple comparison test grouped by experimental condition. Lower values of the NCD indicate that the ERK signal is more predictive of the cellular velocity pattern. For the PIK3CA\_H1047R and PIK3CA\_E545K mutations, there is a significant increase in the relationship between ERK Signaling and cellular velocity ($p<0.003$). These results agree with previous (non-metric) embedding results \supercite{Jacques2021}.

The pattern differences among the six different mutations can be seen in Supplementary Figure \ref*{Fig.Supplment.ssf_kymos}, where each row represents a different mutation. 

\subsection{ERK signaling in colonies of human stem cells}
\label{ERK_signaling_HSC}
Live movies showing colony development of human stem cells (hSCs) were captured. Self-renewing cells are hSCs that divide to produce other hSCs, differentiated cells are formed as the cells progress towards a neural fate. The movies were labeled with ERK-KTR and H2B (nuclear reporter). Ten 2-D movies of self-renewing colonies and ten of differentiated colonies were captured. Supplementary Figure \ref*{Fig.Supplement.HSC_kymos} shows 2-D projections of the ERK SSF for the 10 differentiated movies (A) and the self-renewing movies (B).  The cluster structure function (CSF) for these 20 movies was $[1.1,0.2]$, and the embedding did not indicate clear separation between ERK signaling patterns for self-renewing \emph{vs.} differentiated movies. Additionally, there was no statistically significant change in the relation between ERK and cellular velocity between the self-renewing and differentiated colonies, as shown in Figure \ref*{fig:veloPlot} (B). This result agrees with a biological ground truth that will be published separately.

\begin{figure*}[!hbt]
    \includegraphics[width=\textwidth]{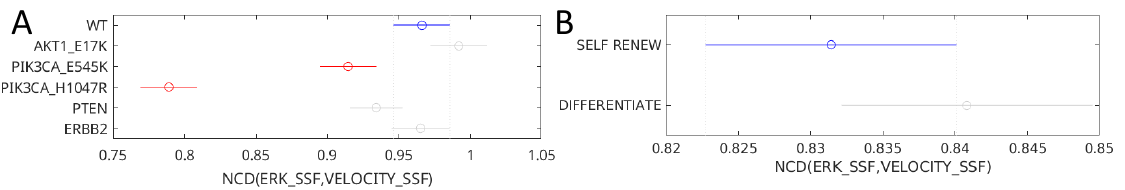}
    \centering
    \caption[]{\textbf{Quantifying the relationship between ERK signaling dynamics and cellular velocity patterns.} Multiple comparison plots showing normalized compression distance between SSF outputs for MCF10A human breast epithelial cell movies grouped by mutation (A) and human stem cells (hSCs) grouped by self-renewing vs. differentiated  (B). ERK and normalized velocity SSF images are computed for every cell in every movie. The normalized information distance  (NID) between ERK and velocity SSF images measures how informative the ERK dynamics are in representing the velocity. Lower NID values (as in e.g. PIK3CA\_H1047R) indicate that there is a stronger relationship in that condition between ERK signaling patterns and cellular velocity. Whiskers indicate 95\% confidence intervals from standard error of the mean. 148 movies from 5 experiment dates were analyzed for the MCF10A cells (A), 10 movies each from self-renewing and differentiated were analyzed for hPSCs (B). In the MCF10A cells, the PIK3CA\_H1047R and PIK3CA\_E545K show significant increases in the relationship between ERK and velocity compared to wild type. In the hPSC cells, there was no significant change in the relationship between ERK and velocity between self-renewing and differentiated cells.}
    \label{fig:veloPlot}
  \end{figure*}

\subsection{Optogenetic excitation of 3-D+time human breast epithelial spheroids}

Optogenetic excitation of ERK signaling of 3-D breast spheroids during live imaging is used to quantify the relationship between ERK signaling patterns and the resulting cellular velocity. Ten movies were captured, showing each 3-D spheroid prior to and during optogenetic excitation. These movies were used previously \supercite{Ender2022}, finding when cells are pulsed with optogenetics (from 4 to 8 hours), cells do less collective rotation. Here we explore more closely the relation between ERK signaling patterns and the resulting cellular velocity. The ERK and velocity SSF images are shown in Supplementary Figure \ref*{Fig.Supplement.opto_kymos}. Using the same method as in section \ref*{sect.mcf10a_results} finds that the spatial ERK signaling pattern is highly linearly correlated with the resulting cell velocity pattern ($\rho=0.99, p<1e-12$), as shown previously \supercite{Gagliardi2023,Ender2022}. 

Ten movies of MCF10A spheroids were imaged at 5 minutes per frame in 3-D (512x512x135 voxels) for $\sim$12 hours (144 frames). ERK and H2B were imaged. Eight movies had four hours of non-excitation, followed by four hours of excitation every 30 minutes and then four hours of non-excitation. Movie 2021719\_6h had six hours of non-excitation, followed by six hours of excitation every 30 minutes and then two hours of non-excitation. Movie 20210721\_2h had two hour intervals for pre-, during- and post-excitation. Each movie had separate SSF images for pre- and excitation.  Figure \ref*{Fig.opto} shows average velocity pre- and excitation. Supplementary Figure \ref*{Fig.Supplment.ssf_kymos} shows the 20 pre- and excitation movies in the NCD-RKHS. Each movie is represented by one point for pre- and one for excitation. In Figure \ref*{Fig.opto} (A) and (B), note the lines connecting these points. The length of these lines represents the visual difference in signaling patterns between each pre- and excitation movie pair as measured by the NCD. We find that the length of each line in the ERK RKHS is highly linearly correlated with the length of the same line in the velocity RKHS ($\rho=0.99, p=1e-12$). Movie 20210720\_4h imaged a spheroid that shifted partially out of frame between pre- and excitation. The resulting distances in the embedded space were large due to the change in visual appearance. Interestingly, the magnitude of the Euclidean distances in the RKHS for velocity and ERK still follow the same linear relationship seen with the other nine movies.  The RKHS here is key in enabling the Euclidean distance in the embedded space to measure quantitative relationships on the input image space.


\begin{figure*}[]
    \centering
    \begin{subfigure}[b]{0.31\textwidth}   
      \includegraphics[width=\textwidth]{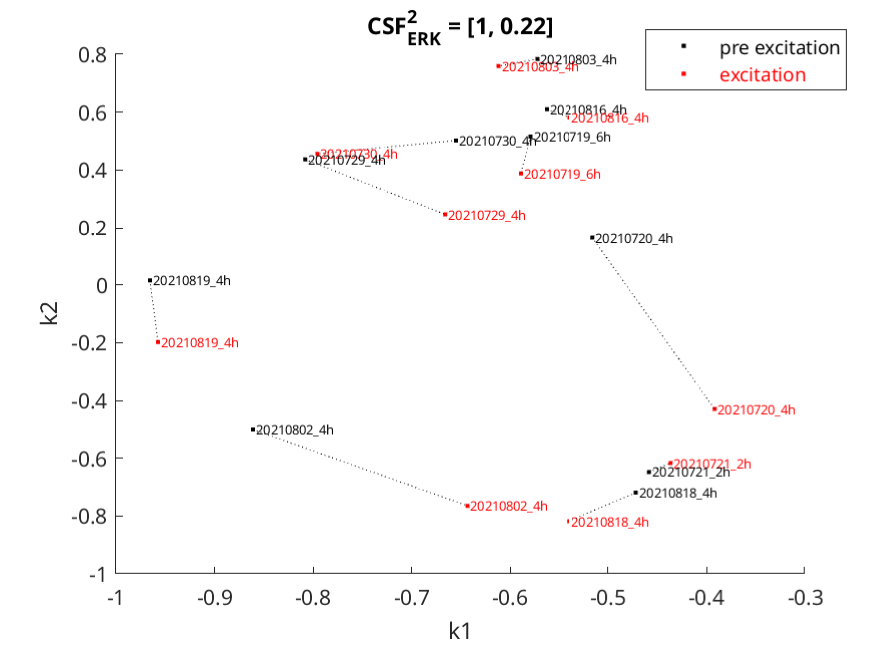}
      \caption{}   
    \end{subfigure} 
    \begin{subfigure}[b]{0.31\textwidth}    
    \includegraphics[width=\textwidth]{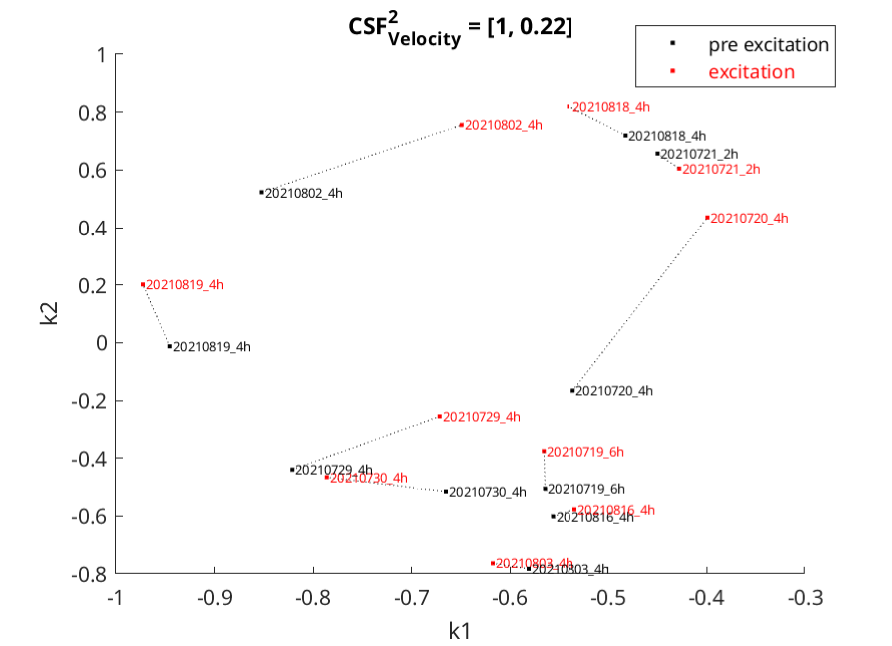}
    \caption{}
    \end{subfigure}
    \begin{subfigure}[b]{0.31\textwidth}    
    \includegraphics[width=\textwidth]{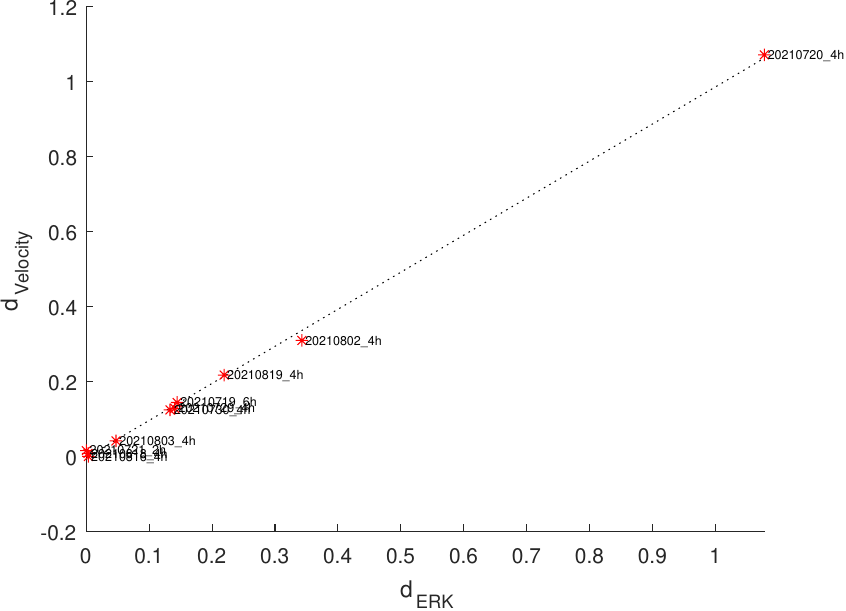}
    \caption{}
    \end{subfigure}
   
    \caption{ 
      \textbf{The Kolmogorov embedding measures correlation between ERK spatiotemporal signaling patterns and the associated cellular velocity}. Each of ten movies $X_1..X_{10}$ are split into 20 pre-excitation and excitation movies $\{X_1^{pre},X_1^{excite}...X_{10}^{pre},X_{10}^{excite}\}$. We project the 3-D+time SSF images to 2-D+time for ERK and velocity for each movie and then compute the pairwise NCD matrix and its resulting embedding in a 2-D RKHS $(k1,k2)$. Because there are two classes (pre- and during-excitation), the CSF value for each embedding is computed in the 2-D space using the ground truth labels to show slightly more structure from the ERK (a) \emph{vs.} velocity (b) embeddings. In (A, B), lines are drawn between pre- and post- excitation image pairs. The linear correlation between the difference magnitude for each pre- vs. excitation movie in ERK vs. velocity (c) is $\rho=0.99$ ($p=1e-12$). The change in ERK signaling pattern induced by optogenetic manipulation is strongly predictive of the resulting change in cellular velocity. 
    }
    \label{Fig.opto}
    
  \end{figure*}

\subsection{Synthetic spatiotemporal signaling patterns} 
Synthetic SSF images simulating 2-D+time constant velocity cellular motion for three simulated classes were created to characterize spatiotemporal pattern extraction. Each synthetic SSF comprises ten cell trajectories moving with randomly generated constant velocity per trajectory normally distributed with a per class mean $\in [1,3,5]$ and standard deviation $\sigma=0.5$. Figure \ref*{fig:phantomKymographs} shows maximum intensity projections of sample phantom SSF image for each velocity class (A, B, and C). At each cell centroid the velocity value is recorded on one channel and a random value $\in [1,255]$ is recorded on a second channel. 100 synthetic SSF images were generated per class.

 Pairs of the synthetic 3-D SSF images are input to the FLIF 3-D compression to compute the pairwise NCD matrix, first using the  velocity channel. This is repeated with the random value channel to create two RKHS embeddings, one for constant velocity values and one for random values. The cluster structure function (CSF) is computed in the RKHS embeddings, preserving the principal three dimensions of the RKHS embedding as there are three ground truth classes. The CSF is computed against these ground truth classes. For the constant velocity SSF images (D), the CSF is $[0.91,0.052]$ (the first value is the mean CSF across the three clusters, the second value is the standard deviation). For the random value SSF images (E), the CSF is $[1.3,0.24]$, indicating that embedding preserves some class structure but captures less meaningful information compared with the constant velocity SSF images.

\begin{figure*}[!hbt]
  \centering 
  \includegraphics[width = \textwidth]{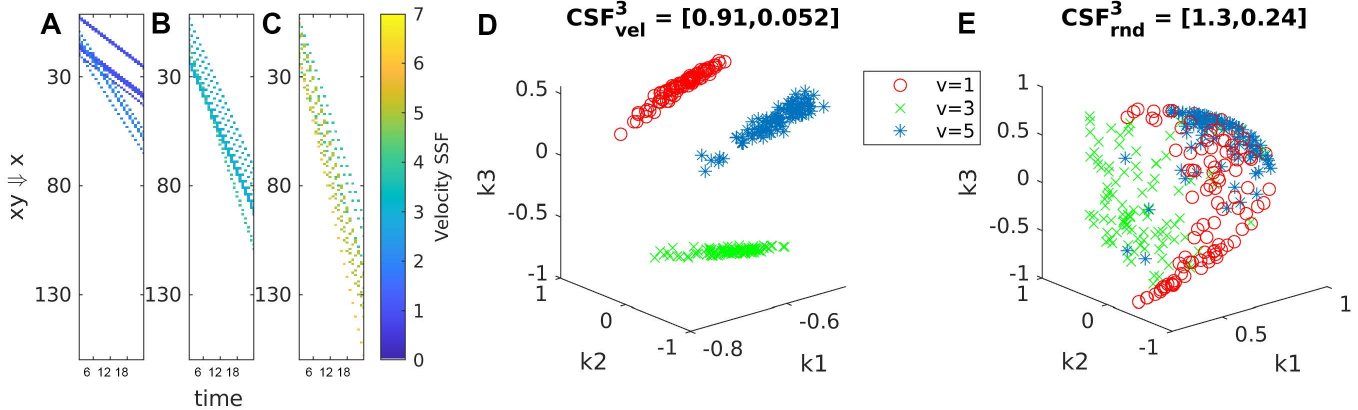}   
  \caption[]{\textbf{Synthetic SSF images characterize spatiotemporal pattern extraction.} SSF images were generated from simulated 2-D+time cellular motion for constant velocity tracks. These SSFs store the velocity on channel one and a random value on channel two at each cell centroid. Three different velocity classes were simulated (A, B, and C). This embedding for the velocity signaling channel (D) shows clear separation among the classes. The random channel embedding (E) preserves some class structure but is a significantly less meaningful embedding as measured by the cluster structure function (CSF). The CSF is reported as the mean and standard deviation of the per cluster optimality deficiency against the ground truth labels.
  }
  \label{fig:phantomKymographs}
\end{figure*}

\section{Discussion}
\label{Sect.Discussion}

The Kolmogorov embedding described here defines a metric pattern space for a collection of 5-D live cell microscopy movies. The normalized information distance (NID) kernel function defines the entire embedding space using lossless 3-D image compression to capture spatio-temporal pattern similarity. Treating non-metric embeddings such as the cytonuclear ratio or neural network latent spaces as metric spaces can introduce statistical bias among different subpopulations based on either physical characteristics or training data and should be avoided. The key advantage of a metric (RKHS) embedding is that the Euclidean distance between any points throughout the embedding space optimally, by decreasing distance matrix eigenvalue, represents the true difference between the corresponding points in the input space.  

The results presented here are intended to illustrate the types of measurements that can be made accurately via metric embeddings. For example, the cluster structure function is more easily computed in the embedding space compared to the input image space to score how meaningfully a particular embedding captures either ground truth (supervised) or clustered (unsupervised) object labels. We do not know of any existing metric embeddings for 5-D live cell signaling movies. The results here agree generally with human visual and other biological (non-metric) characterizations. Due to the lack of quantitative ground truth for the applications considered in the present work we utilize here only unsupervised analyses. One goal in the present work is to demonstrate the utility and possibility of metric embeddings for these types of dataset, and to encourage the development of improved and more accurate lossless compression metrics. Another goal is to demonstrate the need for structure functions as metric filters on image data as required for subsequent metric embedding.

For the 2-D human breast epithelial cell (MCF10A) movies, the RKHS embedding captured the pattern differences among six different oncogenic mutations associated with distinct signaling patterns, with the CSF measuring the structure was significantly higher due to ERK compared to AKT. The NCD kernel measured directly the relation  between ERK signaling and cellular velocity patterns, and found significant increases in the relationship for the PIK3CA\_H1047R and PIK3CA\_E545K mutations. Under these mutations, the ERK pathway is associated more closely with cellular velocity changes, an observation that agrees with empirical visualization. For colonies of human stem cells, it was found that there was no significant difference in ERK dynamics between self-renewing and differentiating colonies, and no change in the relation between ERK and velocity. Using optogenetic excitation while filming live in the microscope, with 3-D MCF10A organoids we see that the magnitude of the change in ERK signaling pattern during excitation is almost perfectly predictive of the change in cellular velocity $\rho=0.99$ ($p=1e-12$). Finally, a synthetic dataset using velocity and random signaling shows the sensitivity in measuring the relation between motion patterns and signaling dynamics.


Kolmogorov complexity theory tells us that as we improve our lossless image compression algorithms, the accuracy of the normalized compression distance in approximating the true Kolmogorov distance between objects will improve. Moving forward, an extension of FLIF to 4-D would allow processing of 3-D+time SSF images with no need for dimensionality reduction. Another area for possible improvement is in the quantization step. The current approach is to quantize the SSF output image from a floating point on $[-1.0,1.0]$ to an 8-bit unsigned integer representation on $[1,255]$. In other work using the NCD with time-lapse microscopy movies has benefited from varying the number of quantization symbols\supercite{NM_cellFate,Cohen2009}. An optimization search on the number of quantization symbols for the movies like those processed here would be time consuming but possible in future work. Another computationally demanding but useful feature would be to generate random reference SSF images  by writing random values at the cell centroid locations and iterating to quantify the contribution of motion patterns \emph{vs.} signaling dynamics in the embedding structure as measured by the CSF. 

\section{Software and data availability}
All of the software tools used are available free and open source, see \url{https://git-bioimage.coe.drexel.edu/opensource/ssfCluster}.
The image data together with segmentation and tracking results can be viewed interactively at \url{https://leverjs.net/ssfCluster}. The LEVERSC 4-D WEBGL viewer \cite[]{Winter} renders 3-D SSF and raw images, and the web API also supports downloading metadata and results directly.

\section{Methods}
\label{Sect.Methods}
\subsection{Live cell imaging}
\subsubsection{Human breast epithelium (MCF10A) monolayers}

Wild-type human mammary epithelial cells MCF10A cells were a gift of Joan S. Brugge, Harvard Medical School, Boston, MA. AKT1-E17K, PIK3CA E545K, PIK3CA H1047R knockin, and MCF10A-PTEN deletion (-//-) knockout derivatives of parental MCF10A cell line\supercite{Gustin2009} was a gift of Ben Ho Park, Johns Hopkins University, USA. ErbB2 overexpressing MCF10A cell line was generated by lentiviral transduction of pHAGE-ERBB2 construct (a gift from Gordon Mills \& Kenneth Scott, Addgene plasmid \#116734\supercite{Ng2018}). Transduction was performed in the presence of 8 $\mu$g//ml polybrene (TR1003, Sigma) in MCF10A WT cells already expressing H2B-miRFP703 and ERK-KTR-mTurquoise2 biosensors. Cells were selected with 5 $\mu$g//ml puromycin (P7255, Sigma).

MCF10A cells were cultured in growth medium composed by DMEM:F12 supplemented with 5\% horse serum, 20 ng/ml recombinant human EGF (Peprotech), 10 mg/ml insulin (Sigma-Aldrich/Merck), 0.5 mg/ml hydrocortisone (Sigma-Aldrich/Merck), 200 U/ml penicillin and 200 mg/ml streptomycin. All the experiments were carried out in starvation medium consisting of DMEM:F12 supplemented with 0.3\% BSA (Sigma-Aldrich/Merck), 0.5 mg/ml hydrocortisone (Sigma-Aldrich/Merck), 200 U/ml penicillin and 200 mg/ml streptomycin. Cells were growth factor and serum starved by removing growth medium, washing the monolayers 2 times with PBS and adding starvation media. 

The stable nuclear marker H2B-miRFP703 was a gift from Vladislav Verkhusha (Addgene plasmid \#80001)\supercite{Shcherbakova2016}, and subcloned in the  PiggyBac plasmid pPBbSr2-MCS. ERK-KTR-mTurquoiose2 and ERK-KTR-mRuby2 sequences were synthesized (GENWIZ) by fusing the ERK Kinase Translocation Reporter (ERK-KTR)\supercite{Regot2014} CDS with mTurquoiose2\supercite{Goedhart2012} and mRuby2\supercite{Lam2012} CDSs, respectively. FoxO3a-mNeonGreen sequence was synthesized (GEN- WIZ) by fusing the 1-1188 portion of the homo sapiens forkhead box O3 a (FoxO3a) CDS with mNeonGreen CDS, a green fluorescent protein derived by Branchiostoma lanceolatum\supercite{Shaner2013}. ERK-KTR-mTurquoiose2 and FoxO3a-mNeonGreen were cloned in the PiggyBac plasmids pMP-PB, pSB-HPB (gift of David Hacker, Lausanne\supercite{Balasubramanian2016}) or pPB3.0.Blast, an improved PiggyBac plasmid generated in Olivier Pertz's lab. For stable DNA integration PiggyBac plasmids were transfected together with the helper plasmid expressing the transposase\supercite{Yusa2011}. To generate cell lines stably expressing nuclear marker and biosensors, transfection was carried out with FuGene (Promega).

Stable clones expressing biosensors were selected using Puromycin (P7255, Sigma), Blasticidin S HCI (5502, Tocris), and Hygromycin B (sc-29067, Lab Force) and imaging-based screening.

MCF10A cells and knock-in/out derivatives were plated on Fibronectin (PanReac AppliChem) coated (0.25ug/cm2)  96 well 1.5 glass bottom plates (Cellvis) at 30 000 cells/ well density and allowed to adhere and form monolayer in growth media. Cells were starved for 48h before starting the experiments. In drug perturbation experiments, starved cells were imaged for 5h, then indicated drugs or vehicle (DMSO) control was added and imaging was resumed for 15h.

Imaging was done on an epifluorescence Eclipse Ti inverted fluorescence microscope (Nikon) controlled by NIS-Elements (Nikon) with a Plan Apo air 20X (NA 0.8) objective. Laser-based autofocus was used throughout the experiments. Image acquisition was performed with an Andor Zyla 4.2 plus camera at a 16-bit depth. Illumination was done with a SPECTRA X light engine (Lumencor) with the following excitation and emission filters (Chroma): far red (miRFP703): 640nm, ET705/72m; red (mRuby2): 555nm, ET652/60m; green/yellow (mNeonGreen): 508nm, ET605/52; cyan (mTurquoise2): 440nm, HQ480/40.

\subsubsection{Human stem cells}

Maintenance 
 WA09 (H9) hESC line was purchased from WiCell (wicell.org) and maintained in Essential 8 flex medium (A2858501, Thermo Fisher Scientific) on hESC-qualified growth factor reduced Geltrex-coated (A1413302, Thermo Fisher Scientific) 6 well plates. Cells were split into 6 well plates at 1:10 ratio when cells become confluent using 0.5 mM EDTA. Medium was changed according to the E8 flex protocol. 
 Time-lapse imaging and differentiation of hESCs
 Established multi-colour hPSC expressing ERKKTR-mClover, ORACLE OCT4tdtomato and H2BmiRFP were plated onto Geltrex-coated 24-well plates (P24-1.5H-N, CellVis) a day before imaging supplemented with Essential 8 flex medium and maintained in the incubator. Colonies were imaged using a Nikon Ti2 with a Yokogawa CSU-W1 spinning disk system. ERKKTR-mClover, ORACLE-OCT4tdtomato and H2BmiRFP were captured every 5min using 488, 561 and 642nm laser respectively using Nis Elements Nd acquisiation modality using 2x2 large image with a 10\% overlap and optimal path blending. To initiate differentiation to neuroectodermal lineage, essential 8 flex medium was changed to PSC neural induction medium (A1647801, Thermo Fisher Scientific) 4h prior to imaging and changed every other day.

\subsubsection{Optogenetic manipulation of 3-D MCF10A spheroids}
Mammary acini were grown from wild-type human female mammary epithelial MCF10A cells. The cells were stably modified by using the PiggyBac transposon system to express H2B-miRFP703, ERK-KTR-mRuby2 and Lyn-cytoFGFR1-PHR-mCit (OptoFGFR), as previously described\supercite{Ender2022}. For acini formation, MCF10A single-cell suspensions were mixed with 4 volumes of growth factor-reduced Matrigel (Corning) at 4$^{\circ}$ C and spread evenly on the surface of glass bottom cell culture plates at a concentration of $1.4 x 10^4 \frac{cells}{cm^2}$. The acini were cultured in DMEM/F12 supplemented with 2\% horse serum, 20 ng/ml recombinant human EGF, 0.5 mg/ml hydrocortisone, 10 mg/ml insulin, 200 U/ml penicillin and 200 mg/ml streptomycin. Horse serum, insulin and EGF were removed after 3 days of culture.  For live imaging, 25 mM Hepes was added to the medium prior to mounting on the microscope. Images of acini were acquired on an epifluorescence Eclipse Ti2 inverted fluorescence microscope (Nikon) equipped with a CSU-W1 spinning disk confocal system (Yokogawa) and a Plan Apo VC 60X water immersion objective (NA = 1.2). For time-lapse imaging, laser-based autofocus was used. Images were acquired with a Prime 95B or a Prime BSI sCMOS camera (both Teledyne Photometrics) at 16-bit depth. Temperature, CO2 and humidity were controlled throughout live imaging with a temperature control system and gas mixer (both Life Imaging Services). The following lasers were used for excitation: 638 nm for far red/miRFP  and 561 nm for red/mRuby2. For the optogenetic stimulation with OptoFGFR, acini were illuminated with wide field blue light (470 nm LED) for 100 ms at 50\% LED intensity at defined time points during spinning disc time-lapse imaging. Both microscopes were controlled by NIS elements (Nikon).

For the optogenetic stimulation with OptoFGFR, acini were illuminated with wide field blue light (470 nm LED) for 100 ms at 50\% LED intensity at defined time points during spinning disc time-lapse imaging.  Live z-stack images of individual acini were acquired every 5 minutes with a 0.6 µm z resolution.

\subsection{Segmentation and tracking}
\label{sect.SegTrack}

The segmentation and tracking used here are based on our previously developed LEVER (lineage editing and validation) tools for segmentation and tracking\supercite{LEVER,MAT,SCR_aging,SCR_lever,TMI_pixelRep,NM_ctc,NM_particles}. The segmentation includes a non-local means denoising\supercite{DIP_MATLAB}, followed by a thresholding and a separation of touching cells. The thresholding uses the SSF values from the H2B channel with an empirical threshold of 0.01 that is used across all of the movies analyzed to date. The thresholding and cell detection uses the H2B channel to identify cell centroids. For each additional image channel (e.g. ERK-KTR, AKT-KTR), we add the maximum magnitude of the positive and negative LoG responses from the thresholded regions identified on the H2B channel, (as in section \ref*{Sect.computeSSF}) and then use a watershed transform on the combined SSF channel images to separate among touching cells. This allows the KTR signal, if available, to assist in the most challenging of the segmentation tasks, separating adherent cells. The multi-temporal association tracking (MAT) \supercite{MAT,NM_particles} can be used to filter cells for inclusion in the SSF output image, but is used here primarily to automatically identify and correct segmentation errors \supercite{NM_cellFate,LEVER,SCR_lever}. The key outcome from the segmentation and tracking are the $(x,y,z,time)$ centroid locations for each cell. In the present work, segmentation and tracking is used to identify cell centroids at which to evaluate the metric structure enhancing filters and for computing cell velocity images. The approach is unsupervised, requiring no training data and taking as its only parameter a range of radii to be used with the LoG GPU filter to compute the SSF (section \ref*{Sect.computeSSF}), set at $[4:0.5:6] \mu m$ for all of the human stem and cancer cells movies analyzed here.

\subsection{Computing the cell signaling structure function}
\label{sect.computeSSF}
\label{Sect.computeSSF}
The cell signaling structure function (SSF) measures the cell signaling state as the intensity of the nuclear voxels w.r.t. the surrounding cytoplasm, useful for any cell imaging protocol, including the powerful new collection of biosensors known as kinase translocation reporters (KTRs)\supercite{Regot2014}. Current state-of-the-art approaches to computationally analyzing these KTR signals rely on the ratio of cytoplasmic to nuclear intensity, the cytonuclear ratio. The cytonuclear ratio is nonlinear, making it a poor choice for reporting cell signaling state. This limitation has been noted in previous work, requiring careful selection of acceptance regions for the KTR signals. An alternative approach to computing KTR activation follows from associating image channel intensities with reporter concentrations. Figure \ref*{fig:ssf} compares the SSF to the cytonuclear ratio. We consider images with intensity values normalized to [0,1]. Given the average cytoplasmic pixel intensity $c$ and the average nuclear pixel intensity  $n$, define the cell signaling structure function $H_{SSF}$ as in eqn. \ref*{eqn.SSF_1D}. The function $H_{SSF}$ is on [-1.0,1.0], and varies linearly with the nuclear and cytoplasmic intensities. In addition to having a linear response to variations in cytoplasmic and nuclear reporter concentrations, $H_{SSF}$ is equal to zero for cells that are unlabeled ($n$ = 0 and $c$ = 0). The $H_{SSF}$ can be robustly approximated via convolution with a Laplacian of Gaussian (LoG) filter. The LoG filter simultaneously estimates $n$ and $c$ while computing Eqn. \ref*{eqn.SSF_1D}.  The value of $H_{SSF}$ is found at each cell centroid and unlike the nuclear/cytoplasmic ratio does not require an accurate segmentation boundary. Finding the cell centroid is an easier task compared to finding the nuclear boundary \supercite{NM_ctc,TMI_pixelRep}.

The Laplacian of Gaussian (LoG) is a widely used blob enhancing filter \supercite{DIP_MATLAB}. The LoG combines the smoothing of a Gaussian filter with the edge enhancement of the Laplacian operator. We define the SSF using the LoG response on each imaging channel $\lambda \in \Lambda$ evaluated at each cell centroid and radius. Here we use the LoG response combined across all imaging channels for the segmentation. Any negative LoG response on the H2B channel is considered foreground. Details of the segmentation are given in the Methods section. The cell signaling structure function $H_{SSF}$ can be written as 

\begin{equation}
    H_{SSF}(\lambda\in\Lambda)=\frac{LoG(x_c,y_c,z_c,\lambda,t,r_{LoG})}{{LoG}_{ref}(r_{SEG})}.
    \label{eq:hSSF}
\end{equation}

$LoG_{ref}$ is the maximum value obtained filtering a zero-intensity spheroid of radius $r_{SEG}$ against a full intensity background. The locations $x_c,y_c,z_c,t$ are cell centroids identified by the segmentation algorithm. In practice, we omit the normalization by $LoG_{ref}$ as this is a constant term due to the multi-resolution LoG implementation as described below, and allow the normalization to [-1.0,1.0] to happen during the quantization step as in Section \ref*{Sect.Quantization}.

The Laplacian of Gaussian filter is a blob enhancing filter that combines a Gaussian smoothing with a Laplacian edge detection. We recently developed a GPU-based implementation of the Laplacian of Gaussian filter that works in 3-D using NVIDIA's CUDA parallel programming toolkit \supercite{HIP}. This filter is scale invariant, meaning that its output remains similar across different radii. This implementation uses axis-aligned spherical approximations to compute the blob response efficiently at every voxel in the image. It is possible to compute the LoG response for arbitrary elliptically oriented blobs, but the extra computational requirements have not been needed for the applications considered to date. 
The LoG filter takes a single parameter of radius. The radius of the filter relates to the standard deviation of the underlying kernel as $r=\sigma\ast\sqrt d$ where $d$ is 2 for 2-D images and 3 for 3-D images. Define a Gaussian kernel $G(x,y,z)$, 

\begin{equation}
    G(x,y,z)=\frac{1}{\sqrt{(2\ \pi)^d\sigma_x^2\sigma_y^2\sigma_z^2}}\exp^{-\frac{1}{2}((\frac{x}{\sigma_x})^2+(\frac{y}{\sigma_y})^2+(\frac{z}{\sigma_z})^2)}.
\end{equation}

Then write the scale invariant Laplacian of Gaussian as 

\begin{equation}
    LoG(x,y,z)={((x,y,z)}^T\Lambda^{-2}(x,y,z)-d)G(x,y,z)
    \label{eq:logInvariant}
\end{equation}

where

\begin{equation}
    \Lambda=\left[\begin{matrix}\sigma_x&0&0\\0&\sigma_y&0\\0&0&\sigma_z\\\end{matrix}\ \ \right] .
\end{equation}

\ref*{eq:logInvariant} is efficient to compute because we omit the covariance terms so $LoG(x,y,z)$ can be computed as a combination of 1-D LoG and Gaussian kernels across the d-dimensional image as 
\begin{multline}
    LoG(x,y,z)=LoG(x)G(y)G(z) 
    +LoG(y)G(x)G(z) 
    +LoG(z)G(y)G(x) 
\end{multline}

This implementation is faster compared to computing a full 3-D kernel as would be required for non-diagonal covariance matrices representing non-axis aligned ellipses. The response $LoG(x,y,z)$ is normalized so that the kernel always sums to zero (even for kernels that protrude from the image) reducing filtering artifacts at image boundaries.

\subsection{Quantizing SSF output images}
\label{Sect.Quantization}
The \emph{cell signaling structure function} (SSF) is a metric function (see Section \ref*{Sect.computeSSF}) that inputs a 3-D image and a 3-element spatial scale parameter representing the radius on each axis. At each cell centroid location, the SSF value for that cell at the given spatiotemporal location and channel is written. For 2-D+time movies, the 3-D SSF output is an exact representation of the 2-D+time movie. For 3-D+time movies, as in the optogenetic excitation of 3-D human breast epithelial cell (MCF10A) spheroids, a maximum intensity projection (MIP) along the $z$ axis reduces the spatial dimensions to 2-D so the SSF output images can be compressed using the 3-D FLIF compressor. We use the MIP on the $z$ axis because of the lower resolution in this direction due to imaging anisotropy. This 3-D output image is then one of two inputs to compute the NCD using the FLIF lossless 3-D image compression. 

The FLIF compression requires 8-bit input images, so the SSF output images must be quantized. Quantization is also required for color mapping the SSF values for visualization, as in Figure \ref*{fig:figure1} (B) and (C) and the supplementary figures. For a given experiment containing $N$ 5-D movies, we generate $N$ 3-D SSF images and extract all non-zero SSF voxels into a multiset $V$. The quantization bins are then defined as 254 linearly spaced boundaries on $(\mu(V)-\sigma(V),\mu(V)+\sigma(V))$, where $\mu(V)$ is the average SSF value and $\sigma(V)$ is the standard deviation. The quantization maps the SSF values to $[1,255]$, reserving the value of zero to indicate an absence of signaling. When assigning a colormap to the the quantized value, the `parula' colormap \supercite{DIP_MATLAB} that is widely available has a bimodal appearance helpful in visualizing the 2-D SSF as a combined color signal and has been used throughout. The FLIF compression algorithm is color agnostic, relying only on entropy calculations on the 8-bit quantized SSF imgaes for pattern detection. The FLIF compression also supports RGB input images, allowing us to compute the NCD between movies with up to three imaging channels simultaneously, but the results to date only utilize single channel inputs to the NCD.

\subsection{The cluster structure function in the RKHS}
\label{Sect.CSF_methods}
The cluster structure function measures how meaningfully a given partition into clusters represents an input dataset. The theory is Kolmogorov complexity, an absolute measure of information content within and between digital objects. For brevity, we show here only the compression approach used to compute the CSF and omit the theory background, for details consult the recent papers and the textbook\supercite{Vitanyi2008,Cohen2023,NCDM}. Our input is a collection of $N$ 5-D $(x,y,z,channel,time)$ microscopy movies $X = (x_1,...,x_N)$. Given a partitioning of $X$ into $K$ clusters as $X = {Y_1,...,Y_K}$, with the $i^{th}$ cluster $Y_i = (y_1,...y_m)$. For each $y_i$, write the optimality deficiency $\delta(Y,y_i)=Z(Y)+log(|Y|)-Z(y_i)$, where $Z(Y)$ is the size in bytes of the compressed $Y$ and $|Y|$ is the cardinality of Y. The average of the optimality deficiency for each of the $m$ elements in the cluster is computed, and that average is taken again across all $K$ clusters to compute a mean and standard deviation of the per cluster optimality deficiencies. When used for 2-D images or smaller 3-D images, the approach is to minimize this CSF in order to select the optimal value of $K$, addressing in an absolute sense  the optimal number of clusters in the given dataset. 

We compute and use the CSF in this work using a different approach than originally proposed\supercite{Cohen2023}. For the 2-D MCF10A human breast epithelium monolayer movies, there are 24 or 25 movies from each genetic condition. To compute $Z(Y)$ for this data would require compressing the concatenation of all those movies, intractable due to the image sizes involved. Instead of compressing the movies in each cluster together, we use the pairwise NCD results to embed the movies in an RKHS and then  compute the CSF in the RKHS as follows. For each cluster in the RKHS $Y'=(y'_1,...y'_m)$, we compute the centroid of the points in that cluster as a point in $\mu_{Y'}=\frac{\sum{y'_i}}{m}$. For each optimality deficiency, we use the Euclidean distance between each of the points representing individual movies for the optimality deficiency, $\delta(Y',y_i') = |y_i-\mu_{Y'}|$. We compute the CSF for every input movie, and then test  statistically significant differences between CSF values, as in ERK vs. AKT for the 2-D MCF10A human breast epithelial cells, using the non-parametric Wilcoxon signed-rank test for significant differences of paired medians \supercite{Theodoridis2009} for the $p$-values reported throughout.

Our application of the CSF in the present work is also used to evaluate how meaningful a given embedding is, rather than the originally proposed use for selecting the optimal number of clusters. Instead of clustering in the RKHS space, we use the ground truth labels for each movie as the partitioning of the data, and choose as $K$ the true number of classes from the experimental conditions. Then for the given ground truth partitioning and $K$ we compute the CSF as a criterion function to evaluate the effectiveness of a particular input, \emph{e.g.} as in evaluating how meaningful ERK vs. AKT in inducing a given structuring of the data for the human breast epithelial cells in Section \ref*{sect.mcf10a_results}.

\section{Acknowledgements}
The authors are grateful to Ben Ho Park for providing AKT1E17K, PIK3CAE545K, PIK3CAH1047R knockin, and PTEN(-/-) knockout MCF10A - derived cell lines\supercite{Gustin2009}. The authors are also grateful to Rafael Carazo Salas and Sungmin Kim from the Univ. Bristol, UK, for the imaging constructs used in the hSC dataset.  This project has received funding from the European Union's Horizon 2020 research and innovation program under the Marie Skłodowska-Curie grant agreement No 896310 to Agne Frismantiene. We acknowledge support of the Microscopy Imaging Center of the University of Bern (https://www.mic.unibe.ch/). Portions of this work were supported by Human Frontiers Science Grant RGP0043/2019 (Pertz/Cohen). 

\section{Software and data availability}
All of the software tools used are available free and open source, see \url{https://git-bioimage.coe.drexel.edu/opensource/ssfCluster}.
The image data together with segmentation and tracking results can be viewed interactively at \url{https://leverjs.net/ssfCluster}. The LEVERSC 4-D WEBGL viewer \cite[]{Winter} renders 3-D movies together with the SSF output images. This web API also supports downloading metadata and results directly. 
\label{Sect.Software}

\printbibliography

\newpage
\renewcommand{\figurename}{Supplementary Figure}
\setcounter{figure}{0}  

\begin{figure*}[]
    \centering
    \begin{subfigure}[b]{0.95\textwidth}    
      \includegraphics[width = \textwidth,page=1]{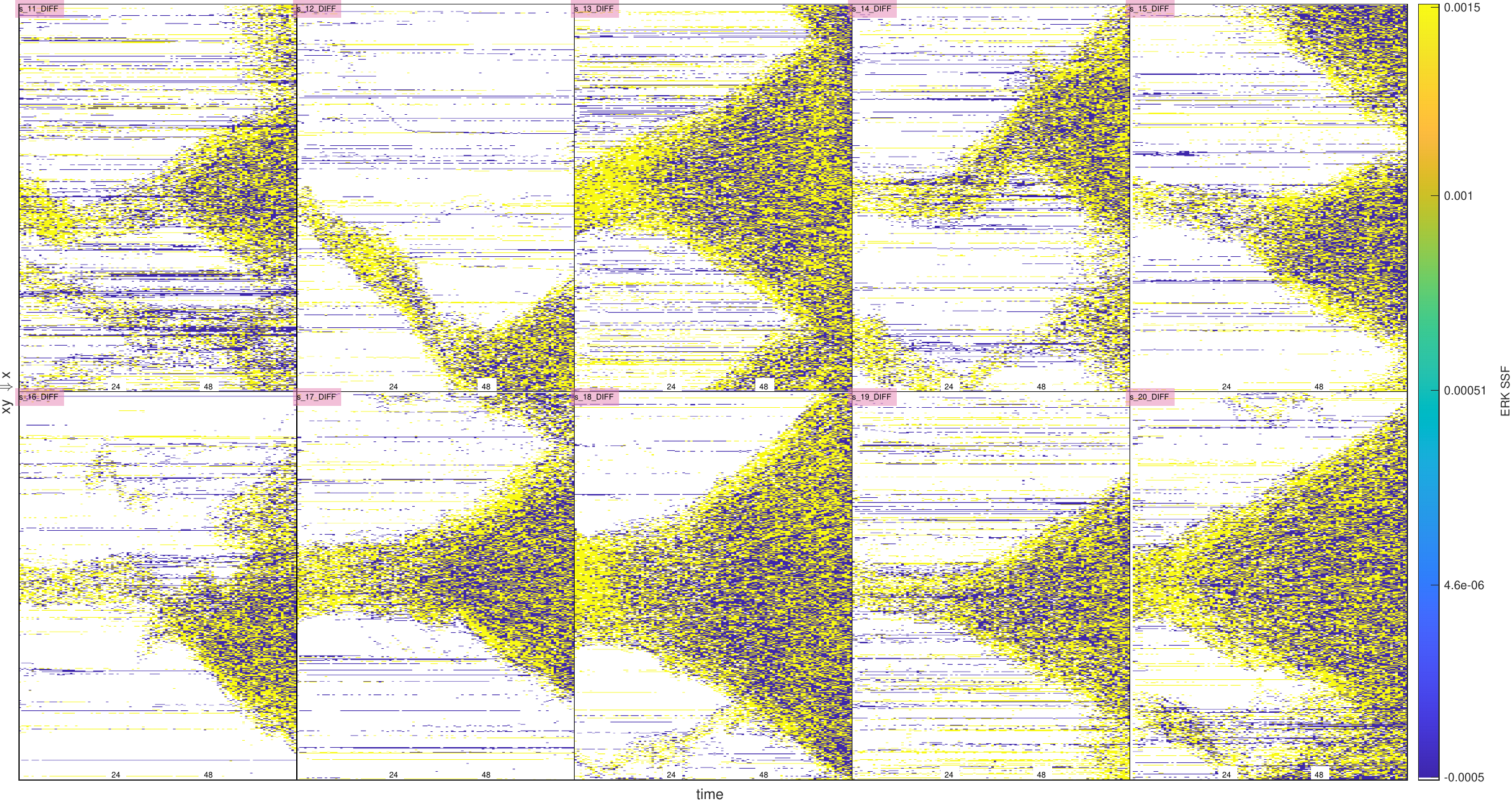}
    \caption{
      }
    \end{subfigure}
    \begin{subfigure}[!htb]{0.95\textwidth}    
      \includegraphics[width = \textwidth,page=2]{figures/hsc_erk_kymos_condition_equiprob.pdf}
      \caption{
      }
    \end{subfigure}
    \caption{2-D projections of 3-D (2-D+time) SSF outputs for movies showing ERK signaling in colonies of human stem cells for 10 differentiated (a) and 10 self-renewing (b) movies. The vertical axis represents the spatial dimension, and is obtained by taking a maximum intensity projection along the second spatial axis. The horizontal axis in each panel represents time. These quantitative visualizations show ERK patterns throughout the process of colony development. The 2-D visualization here is lower dimensional than the 3-D SSF image that is a lossless representation input to the compression distance for embedding. The full dataset can be seen here: \url{https://leverjs.net/ssfCluster/HSC}.
    }
    \label{Fig.Supplement.HSC_kymos}
    \end{figure*}
    
\begin{figure*}[!hbt] 
  \centering
  \begin{subfigure}[b]{0.95\textwidth}  
  \includegraphics[width = \textwidth]{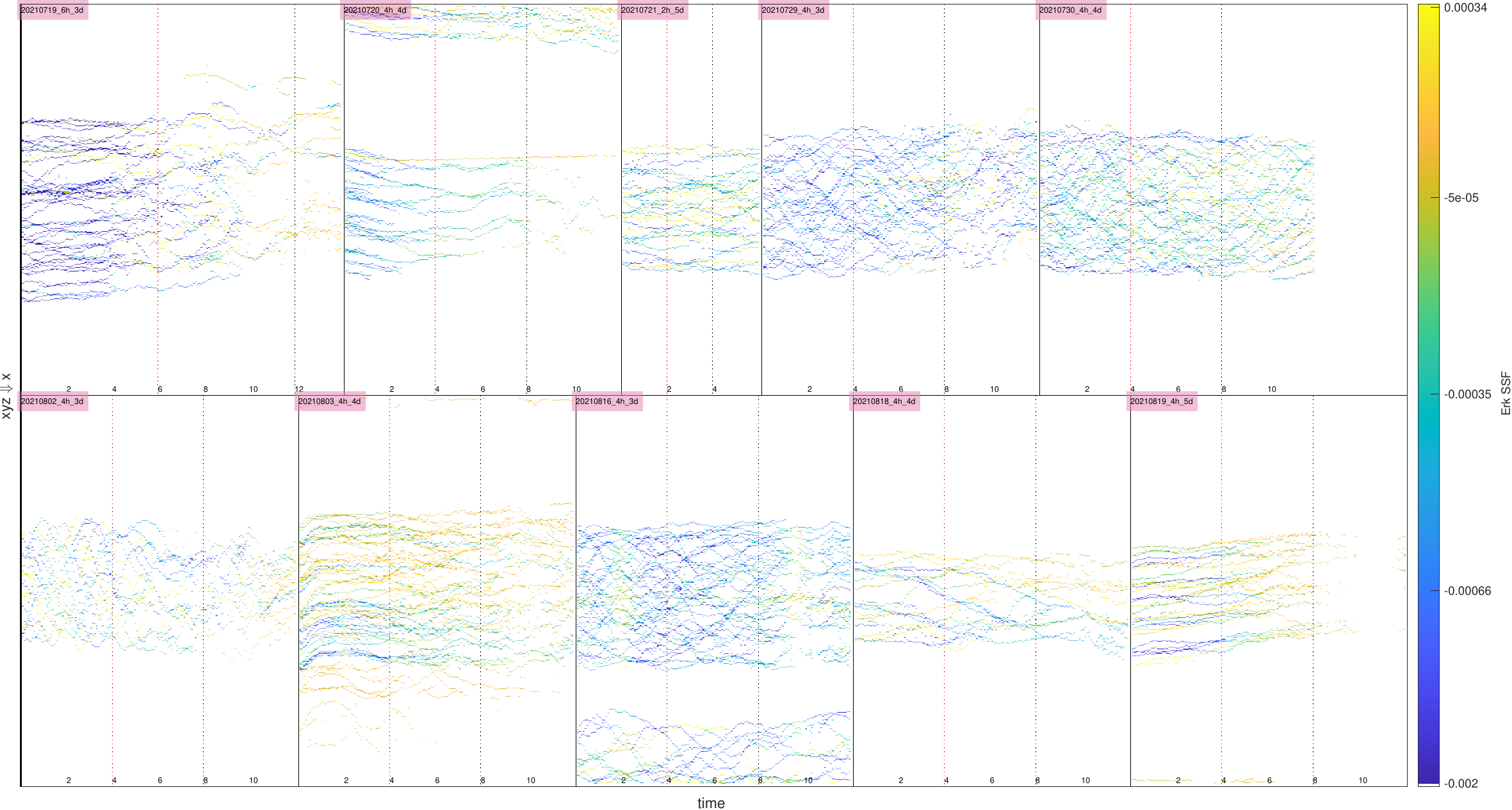}
  \caption{}
  \end{subfigure}
  \begin{subfigure}[b]{0.95\textwidth}  
    \includegraphics[width = \textwidth]{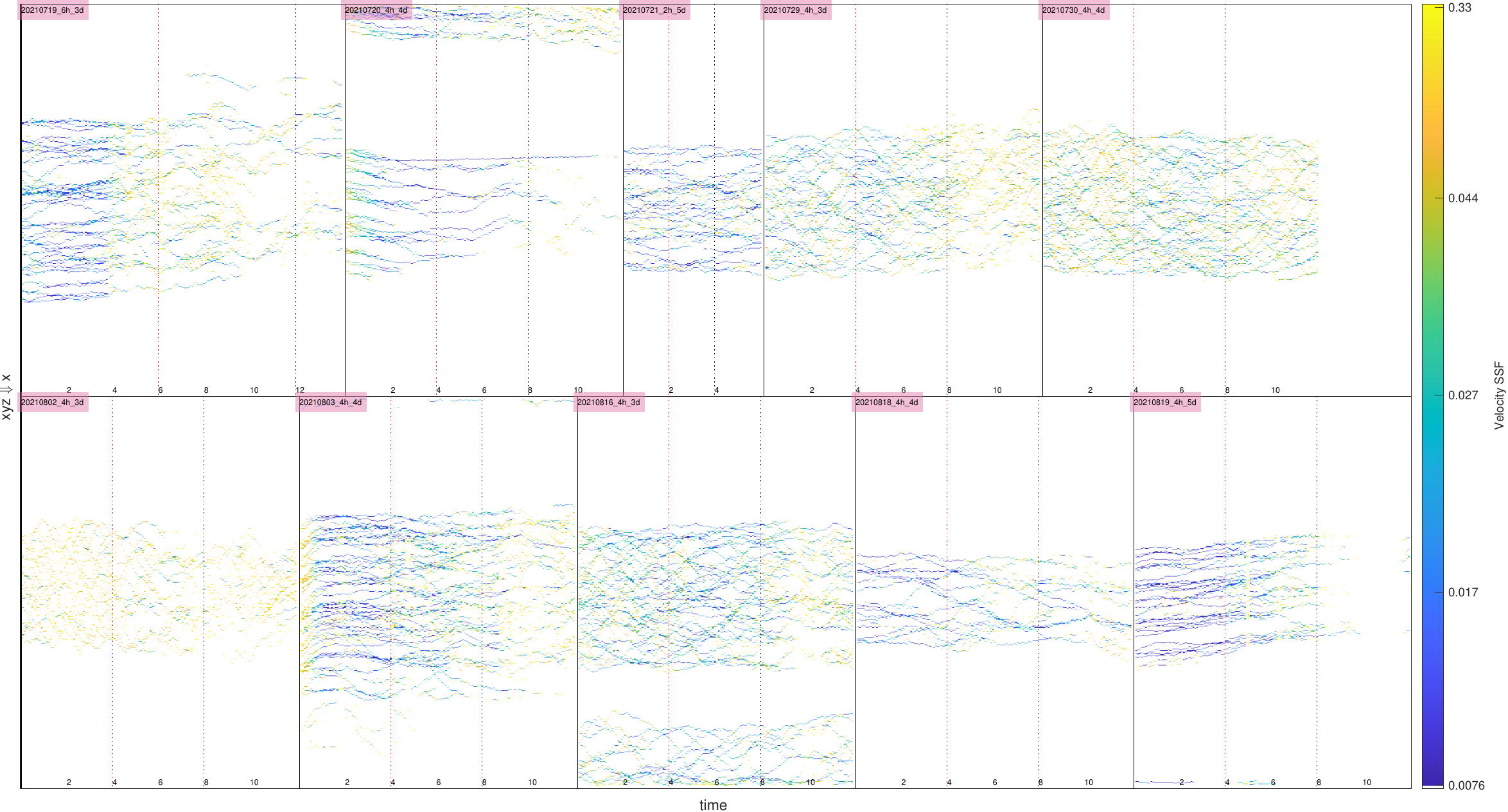}
    \caption{}
    \end{subfigure}
    
  \caption{
  2-D projections of 3-D+time ERK (a) and velocity (b) SSF output images for 3-D+time movies of mammary acini, spheroids of human female mammary epithelial MCF10A cells. The original 3-D+time movies are processed with a 3-D SSF via a maximum intensity projection along the $Z$ spatial axis to form the input to the compression. For the 2-D rendering shown here, the spatial component is again projected via maximum intensity along the $Y$ axis. Each movie is labeled by the time (hours) before optogenetic excitation, the time (hours) that optogenetic excitation lasts (pulses every 30 minutes), and the age of the organoid (days). The dashed vertical lines indicate the beginning and end of the optogenetic excitation. The 2-D projections shown here are useful for human visualization, the 3-D SSF output images are input the FLIF compression algorithm to compute the pairwise NCD to generate the reproducing kernel Hilbert space embedding (\ref*{Fig.opto}). The optogenetic excitation dataset can be seen here: \url{https://leverjs.net/ssfCluster/optoGenetic}. 
  } 
  \label{Fig.Supplement.opto_kymos}
\end{figure*} 

\begin{figure*}[!hbt] 
  \includegraphics*[width=\textwidth]{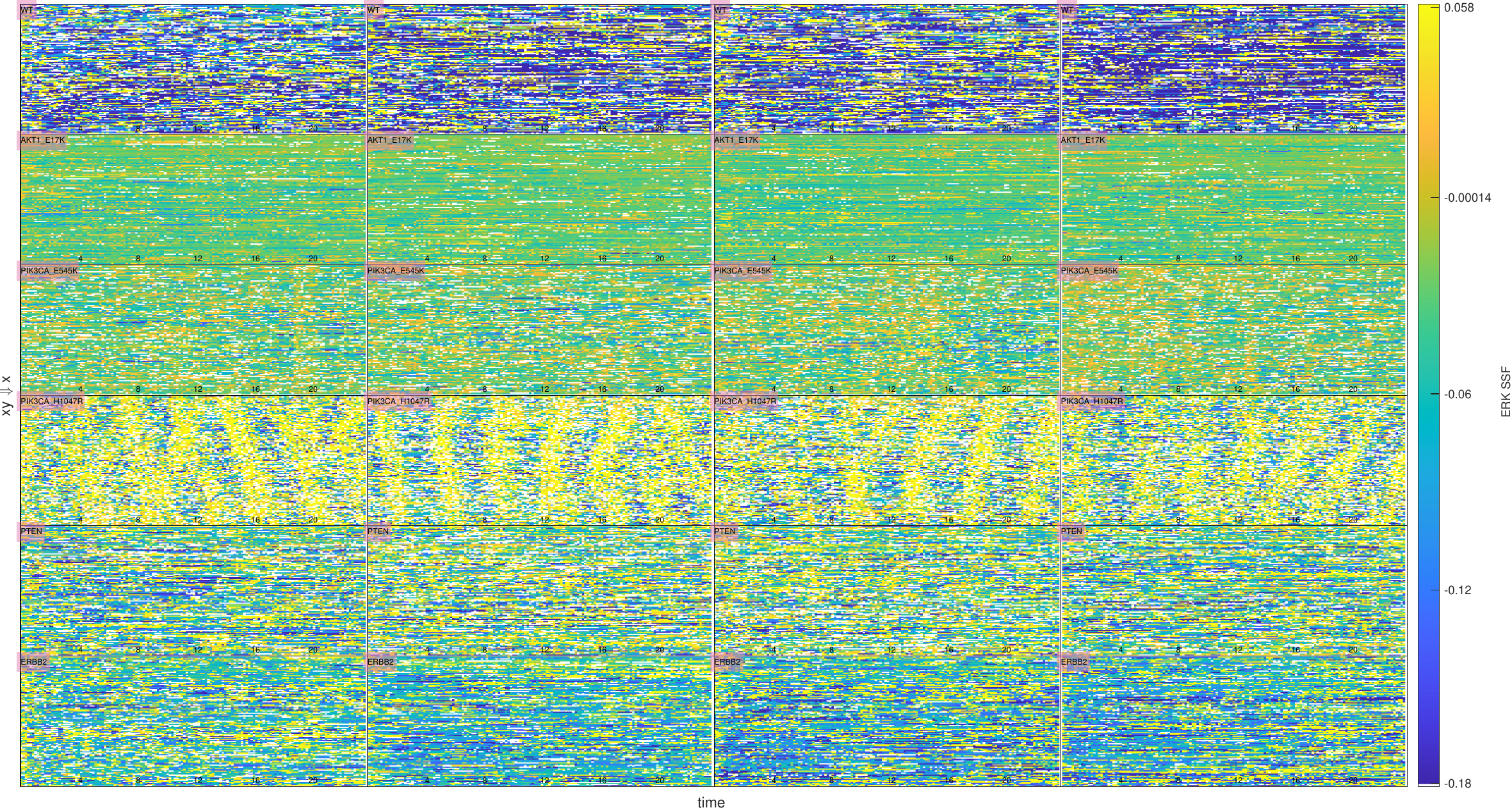}
  \centering
  \caption{
    2-D projections of 3-D SSF output images for 2-D+time movies showing live monolayers of human breast epithelial cells (MCF10A). These 24 movies are from one of six imaging experiments, showing five oncogenic mutations plus wild type (one each per row), also shown as an RKHS embedding in Figure \ref*{fig:figure1} (D). The full dataset contains 147 movies, and can be viewed here: \url{https://leverjs.net/ssfCluster/MCF10A_2D}. 
  }   
  \label{Fig.Supplment.ssf_kymos}
\end{figure*}

\renewcommand{\figurename}{Supplementary Movie}
\setcounter{figure}{0}  
\begin{figure*}  
  \caption{   
  Animated version of Figure \ref*{fig:figure1}. A time-lapse movie showing ERK-KTR signaling in a monolayer of human breast epithelial cells (MCF10A) from the PIK3CA\_H1047R mutation with cellular activation indicated by dark nuclei against bright cytoplasm clearly propagating across the image (left panel). The 3-D SSF metric image filter output (center panel) is the input to the FLIF 3-D compression used with the normalized compression distance to define the RKHS embedding, shown here with the current image frame overlaid in gray. The 2-D projection of the 3-D SSF output image (right) panel facilitates human visualization, with the current timepoint indicated by the red line and the signaling patterns clearly visible as diagonal yellow stripes of activation across the monolayer \href{https://bioimage.coe.drexel.edu/media/ssfClustering/movie1.mp4}{(Link)}.
  \label{Movie1}
  }
\end{figure*}
  
\begin{figure*}  
  \caption{   
  Rotating view of 3-D SSF output image from Figure \ref*{fig:figure1} (B). The 3-D SSF computes the signaling activation  at each cell centroid location $(x,y,time)$. The normalized compression distance finds patterns of similarity between SSF output image pairs. Shown here is an SSF output image for the PIK3CA\_H1047R oncogenic mutation known for distinct ERK signaling patterns,  \href{https://bioimage.coe.drexel.edu/media/ssfClustering/movie2.mp4}{(Link)}.
  \label{Movie2}
  }
\end{figure*}

\end{document}